\documentclass[11pt]{article}

\usepackage{dsfont}

\def\colorful{0}

\usepackage{amsthm,amsfonts,amsmath,amssymb,epsfig,color,float,graphicx,verbatim, enumitem}
\usepackage{multirow}
\usepackage{algorithm}
\usepackage[noend]{algpseudocode}

\newif\ifhyper\IfFileExists{hyperref.sty}{\hypertrue}{\hyperfalse}
\hypertrue
\ifhyper\usepackage{hyperref}\fi

\usepackage{enumitem}

\usepackage{framed}
\usepackage{nicefrac}

\def\nnewcolor{1}
\ifnum\nnewcolor=1

\fi
\ifnum\nnewcolor=0

\fi

\ifnum\colorful=1
\newcommand{\new}[1]{{\color{red} #1}}

\else
\newcommand{\new}[1]{{#1}}

\fi

\newtheorem{theorem}{Theorem}[section]
\newtheorem{question}{Question}[section]

\newtheorem{cond}[theorem]{Condition}
\newtheorem{lemma}[theorem]{Lemma}
\newtheorem{informal theorem}[theorem]{Theorem (informal statement)}

\newtheorem{proposition}[theorem]{Proposition}
\newtheorem{corollary}[theorem]{Corollary}
\newtheorem{claim}[theorem]{Claim}
\newtheorem{fact}[theorem]{Fact}

\newtheorem{remark}[theorem]{Remark}

\theoremstyle{definition}
\newtheorem{definition}[theorem]{Definition}
\newcommand{\eqdef}{\stackrel{{\mathrm {\footnotesize def}}}{=}}

\newcommand{\p}{\mathbf{P}}

\newcommand{\R}{\mathbb{R}}

\newcommand{\Z}{\mathbb{Z}}
\newcommand{\N}{\mathbb{N}}
\newcommand{\E}{\mathbf{E}}
\newcommand{\eps}{\epsilon}

\newcommand{\pr}{\mathbf{Pr}}
\renewcommand{\Pr}{\mathbf{Pr}}
\newcommand{\poly}{\mathrm{poly}}

\newcommand{\sgn}{\mathrm{sign}}
\newcommand{\sign}{\mathrm{sign}}

\newcommand{\opt}{\mathrm{OPT}}
\newcommand{\D}{\mathcal{D}}

\title{Near-Optimal Statistical Query Hardness \\ of Learning Halfspaces with Massart Noise}

\author{
Ilias Diakonikolas\thanks{Supported by NSF Medium Award CCF-2107079,
NSF Award CCF-1652862 (CAREER), a Sloan Research Fellowship, and
a DARPA Learning with Less Labels (LwLL) grant.}\\
University of Wisconsin-Madison\\
{\tt ilias@cs.wisc.edu}\\
\and
Daniel M. Kane\thanks{Supported by NSF Medium Award CCF-2107547,
NSF Award CCF-1553288 (CAREER), and a Sloan Research Fellowship.}\\
University of California, San Diego\\
{\tt dakane@cs.ucsd.edu}\\
}

\begin{document}

\maketitle

\begin{abstract}
We study the problem of PAC learning halfspaces with Massart noise. 
Given labeled samples $(x, y)$
from a distribution $D$ on $\R^{d} \times \{ \pm 1\}$ 
such that the marginal $D_x$ on the examples is arbitrary 
and the label $y$ of example $x$ is generated from the target halfspace 
corrupted by a Massart adversary with flipping probability $\eta(x) \leq \eta \leq 1/2$,
the goal 
is to compute a hypothesis with small misclassification error.
The best known $\poly(d, 1/\eps)$-time algorithms for this problem 
achieve error of $\eta+\eps$, which can be far from the optimal bound of $\opt+\eps$, 
where $\opt = \E_{x \sim D_x} [\eta(x)]$. 
While it is known that achieving $\opt+o(1)$ error requires super-polynomial time 
in the Statistical Query model, a large gap remains between 
known upper and lower bounds.

In this work, we essentially characterize
the efficient learnability of Massart halfspaces in the Statistical Query (SQ) model. 
Specifically, we show that no efficient SQ algorithm for learning Massart halfspaces on $\R^d$
can achieve error better than $\Omega(\eta)$, even if $\opt  = 2^{-\log^{c} (d)}$, 
for any universal constant $c \in (0, 1)$. Furthermore, when the noise upper bound $\eta$ is close to $1/2$, 
our error lower bound becomes $\eta  - o_{\eta}(1)$, where the $o_{\eta}(1)$ term goes to $0$ 
when $\eta$ approaches $1/2$.
Our results provide strong evidence that known 
learning algorithms for Massart halfspaces are nearly best possible, 
thereby resolving a longstanding open problem in learning theory.
\end{abstract}

\setcounter{page}{0}

\thispagestyle{empty}

\newpage

\section{Introduction} \label{sec:intro}

\subsection{Background and Motivation} \label{ssec:background}

A halfspace, or Linear Threshold Function (LTF), is any function $f: \R^m \to \{ \pm 1\}$ of the form
$f(x) = \sgn(w \cdot x- \theta)$, for some weight vector $w \in \R^m$ and threshold $\theta \in \R$.
(The function $\sign: \R \to \{ \pm 1\}$ is defined as $\sgn(t)=1$ if $t \geq 0$ and $\sgn(t)=-1$ otherwise.)
Halfspaces are a fundamental class of Boolean functions that have been extensively studied in
computational complexity and learning theory over several decades~\cite{MinskyPapert:68, Yao:90, GHR:92, CristianiniShaweTaylor:00, AoBF14}. The problem of learning an unknown halfspace is as old
as the field of machine learning, starting with the Perceptron algorithm~\cite{Rosenblatt:58, Novikoff:62},
and has been one of the most influential problems in this field  with techniques
such as SVMs~\cite{Vapnik:98} and AdaBoost~\cite{FreundSchapire:97} coming out of its study.

In the realizable PAC model~\cite{val84}, i.e., when the labels are consistent with the target function,
halfspaces are efficiently learnable via Linear Programming (see, e.g.,~\cite{MT:94}).
In the presence of noisy data, the computational complexity of learning halfspaces depends on
the underlying noise model. Here we study the complexity of learning halfspaces
with Massart noise. In the Massart (or bounded) noise model, the label of each example $x$
is flipped independently with probability $\eta(x) \leq \eta$, for some parameter $\eta \leq 1/2$.
The flipping probability is 
bounded above by $1/2$, but can depend on the example $x$ in a potentially adversarial manner.
The following definition encapsulates the PAC learning problem in this noise model.

\begin{definition}[PAC Learning with Massart Noise] \label{def:massart-learning}
Let $\mathcal{C}$ be a concept class of Boolean-valued functions over $X= \R^m$, $D_{x}$
be a fixed but unknown distribution over $X$, and $0 \leq \eta \leq 1/2$ be the noise parameter.
Let $f: X \to \{\pm1\}$ be the unknown target concept with $f \in \mathcal{C}$.
A {\em Massart example oracle}, $\mathrm{EX}^{\mathrm{Mas}}(f, D_{x}, \eta)$,
works as follows: Each time $\mathrm{EX}^{\mathrm{Mas}}(f, D_{x}, \eta)$ is invoked,
it returns a labeled example $(x, y)$, where $x \sim D_{x}$, $y = f(x)$ with
probability $1-\eta(x)$ and $y = -f(x)$ with probability $\eta(x)$,
for some {\em unknown} function $\eta(x): X \to [0, 1/2]$ 
with $\eta(x)\leq \eta$ for all $x \in X$. 
Let $D$ denote the joint distribution on $(x, y)$ generated by the Massart example oracle.
A PAC learning algorithm is given i.i.d.\ samples from $D$
and its goal is to output a hypothesis $h: X \to \{\pm1\}$ such that with high probability
the error $\pr_{(x, y) \sim D} [h(x) \neq y]$ is as small as possible. We will use
$\opt \eqdef \inf_{g \in \mathcal{C}} \pr_{(x, y) \sim D} [g(x) \neq y]$
to denote the optimal misclassification error.
\end{definition}

A remark is in order about the definition. While the TCS community had only considered
the case that the upper bound $\eta$ on the Massart noise rate is {\em strictly smaller} than $1/2$,
this is not an essential assumption in the model. In fact, the original definition of the Massart model~\cite{Massart2006} 
allows for $\eta = 1/2$. (Note that it is possible that $\eta = 1/2$ while $\opt$ is much smaller.)

The Massart noise model is a natural semi-random input model that was formulated in~\cite{Massart2006}.
An equivalent noise model had already been defined in the 80s
by Sloan and Rivest~\cite{Sloan88, Sloan92, RivestSloan:94, Sloan96}
(under the name ``malicious misclassification noise'') and a very similar definition
had been proposed and studied even earlier by Vapnik~\cite{Vapnik82}.
The {\em sample complexity} of PAC learning halfspaces with Massart noise is well-understood.
For example, it is known (see, e.g.,~\cite{Massart2006}) that for any concept class
$\mathcal{C}$ of VC dimension $d$, $O(d/\eps^2)$ samples
information-theoretically suffice to compute a hypothesis with misclassification error $\opt+ \eps$,
where  $\opt = \E_{x \sim D_{x}} [\eta(x)]$.
This implies that halfspaces on $\R^m$ are learnable in the Massart model with
$O(m/\eps^2)$ samples.

In sharp contrast, our understanding of the {\em algorithmic aspects} of
PAC learning various natural concept classes with Massart noise
is startlingly poor and has remained a tantalizing open problem
in computational learning theory since the 1980s.
In~\cite{Sloan88} (see also~\cite{Sloan92}),
Sloan defined the malicious misclassification noise model
(an equivalent formulation of Massart noise)
and asked  whether there exists an efficient learning algorithm
for Boolean disjunctions, i.e., ORs of Boolean literals --- a very 
special case of halfspaces ---  in this model.
About a decade later, Edith Cohen~\cite{Cohen:97} asked 
the same question for the general class of halfspaces.
The problem remained open, even for weak learning,
and was highlighted in Avrim Blum's FOCS 2003 tutorial~\cite{Blum03}.
Surprisingly, until fairly recently, it was not even known whether 
there exists an efficient algorithm that achieves misclassification error 
$49\%$ for Massart halfspaces with noise rate upper bound of $\eta = 1\%$.

Recent work~\cite{DGT19} made the first algorithmic progress on this learning problem.
Specifically, \cite{DGT19} gave a $\poly(m, 1/\eps)$-time learning algorithm for Massart halfspaces
with error guarantee of $\eta+\eps$, where $\eta$ is the upper bound on the Massart noise rate.
This is an {\em absolute} error guarantee which cannot be improved  in general ---
since it may well be the case that $\opt = \eta$ (this in particular happens when $\eta(x) = \eta$ for all $x \in X$).
Motivated by \cite{DGT19}, more recent work~\cite{DiakonikolasIKL21} gave an efficient boosting algorithm, 
achieving error $\eta+\eps$ for any concept class, assuming the existence of a weak learner for the class.

The aforementioned error bound of $\eta+\eps$ can be very far from the
information-theoretically optimum error of $\opt+\eps$. Recall that
$\opt = \E_{x \sim D_{x}} [\eta(x)] \leq \eta$ and it could well be the case that $\opt \ll \eta$.
Follow-up work by~\cite{CKMY20} showed that {\em exact} learning --- specifically,
obtaining error of $\opt+o(1)$, when $\opt$ is close to $1/2$ ---
requires super-polynomial time in the Statistical Query (SQ) model of~\cite{Kearns:98}.
The latter SQ lower bound is very fragile in the sense that it does not even rule out {\em any} 
constant factor approximation algorithm for the problem, 
i.e., a $\poly(m, 1/\eps)$-time learning algorithm with error $C \cdot \opt+\eps$, 
for a universal constant $C>1$.

The aforementioned progress notwithstanding, a very large gap remains in our understanding of the
efficient learnability of halfspaces in the presence of Massart noise. In particular, 
prior to the current work, the following questions remained open.
\begin{question} \label{q:open}
Is there an efficient learning algorithm for Massart halfspaces achieving 
a {\em relative} error guarantee? Specifically, if $\opt \ll \eta$ is it possible to efficiently achieve 
error  significantly better than $\eta$? More generally, what is the best error (as a function of $\opt$ and $\eta$) 
that can be achieved in polynomial time?
\end{question}

We emphasize here that, throughout this work, we focus on {\em improper learning},
where the learning algorithm is allowed to output any polynomially evaluatable
hypothesis.

\medskip

\noindent In this paper, {\em we essentially resolve the efficient PAC learnability 
of Massart halfspaces in the SQ model.}
Specifically, we prove a near-optimal super-polynomial SQ lower bound for this problem, 
which provides {\em strong evidence that known efficient algorithms are nearly best possible.}

Before we formally state our contributions,
we require some background on SQ algorithms.

\paragraph{Statistical Query Model}
Statistical Query (SQ) algorithms are the class of algorithms
that are only allowed to query expectations of bounded functions 
of the underlying distribution rather than directly access samples.
The SQ model was introduced by Kearns~\cite{Kearns:98} in the context of supervised learning
as a natural restriction of the PAC model~\cite{val84} and has been extensively studied
in learning theory. A recent line of work~\cite{FGR+13, FeldmanPV15, FeldmanGV17, Feldman17}
generalized the SQ framework for search problems over distributions. The reader
is referred to~\cite{Feldman16b} for a survey.

One can prove unconditional lower bounds on the complexity of SQ algorithms
via a notion of {\em Statistical Query dimension}. 
Such a complexity measure was introduced in~\cite{BFJ+:94}
for PAC learning of Boolean functions and has been generalized 
to the unsupervised setting in~\cite{FGR+13, Feldman17}.
A lower bound on the SQ dimension of a learning problem provides an unconditional lower bound
on the complexity of any SQ algorithm for the problem.

The class of SQ algorithms is fairly broad:  a wide range of known algorithmic techniques
in machine learning are known to be implementable in the SQ model.
These include spectral techniques, moment and tensor methods,
local search (e.g., Expectation Maximization), and many others
(see, e.g.,~\cite{Chu:2006, FGR+13, FeldmanGV17}).
In the context of PAC learning classes of Boolean functions (the topic of this paper),
with the exception of learning algorithms using Gaussian elimination
(in particular for the concept class of parities, see, e.g.,~\cite{BKW:03}),
all known algorithms with non-trivial performance guarantees
are either SQ or are implementable using SQs. Finally, we acknowledge very
recent work~\cite{BBHLS20} which established a close connection
between the SQ model and low-degree polynomial tests under certain assumptions.

\subsection{Our Contributions} \label{ssec:results}

Our main result shows that any efficient (i.e., using polynomially 
many queries of inverse-polynomial accuracy) SQ learning algorithm for Massart halfspaces 
on $\R^m$ cannot obtain error better than $\Omega(\eta)$, even if the optimal error
is as small as $\opt = 2^{-\log^c(m)}$, for any constant $c \in (0, 1)$. This result rules out
even very weak relative approximations to the optimal value.

In more detail, we establish the following theorem:

\begin{theorem}[Main Result] \label{thm:main-inf}
For any universal constants $c, c'$ with $0< c < 1$ and $0< c'<1-c$, the following holds.
For any sufficiently large positive integer $m$ and any $0 < \eta <1/2$, 
there is no SQ algorithm that PAC learns the class of halfspaces in $\R^m$ with $\eta$-Massart noise 
to error better than $\Omega(\eta)$ using at most $\exp(\log^{1+c}(m))$ queries of accuracy 
no better than $\exp(-\log^{1+c}(m))$. 
This holds even if the optimal classifier has error $\opt = \exp(-\log^{c'}(m))$.
\end{theorem}

Some comments are in order. First, recall that the efficient algorithm of~\cite{DGT19} 
(which can be implemented in the SQ model)
achieves error arbitrarily close to $\eta$. 
Moreover, it is easy to see that the Massart learning problem is computationally easy when $\opt \ll 1/m$. 
As a result, the ``inapproximability gap'' of $\Omega(\eta)$ versus $2^{-\log^c(m)}$ established 
by Theorem~\ref{thm:main-inf} is essentially best possible (up to the universal constant in the $\Omega(\cdot)$).
For a more detailed statement, see Theorem~\ref{thm:main-formal}. 

\begin{remark} \label{rem:optimal-error-large-eta}
{\em When the Massart noise rate upper bound $\eta$ approaches $1/2$, 
we can replace the lower bound of $\Omega(\eta)$ appearing in Theorem~\ref{thm:main-inf} 
by the sharper lower bound of $\eta  - o_{\eta}(1)$. Here the term $o_{\eta}(1)$ goes to $0$ 
as $\eta$ approaches $1/2$. See Theorem~\ref{thm:main-large-eta} for the statement in this regime.}
\end{remark}

It is worth comparing Theorem~\ref{thm:main-inf} to the hardness result of Daniely~\cite{Daniely16}
for PAC learning halfspaces in the {\em agnostic} model.
Daniely's result is qualitatively similar to our Theorem~\ref{thm:main-inf} with two differences:
(1) The lower bound in~\cite{Daniely16} only applies against the (much more challenging) agnostic model.
(2) In the agnostic setting, it is hard to learn halfspaces within error significantly better than $1/2$,
rather than error $\Omega(\eta)$ in the Massart setting.
Theorem~\ref{thm:main-formal} proves an SQ lower bound for a much more benign semi-random 
noise model at the cost of allowing somewhat better error in polynomial time. \new{We reiterate that
error arbitrarily close to $\eta$ is efficiently achievable for Massart noise~\cite{DGT19}, 
and therefore our hardness gap is nearly best possible.}


\begin{remark}\label{rem:prior-result}
{\em Theorem~\ref{thm:main-inf} strengthens a recent result by the same authors~\cite{DK20-hardness-arxiv} 
which established a weaker inapproximability gap for Massart halfspaces in the SQ model. 
Specifically,~\cite{DK20-hardness-arxiv} showed that no efficient SQ algorithm 
can learn Massart halfspaces to accuracy $1/\mathrm{polylog}(m)$, even when $\opt = 2^{-\log^c(m)}$. 
The main difference between the two results, which leads to the difference in the error bounds, 
lies in the construction of the one-dimensional moment-matching distributions.}
\end{remark}

\subsection{Related and Prior Work} \label{ssec:related}
We have already provided some background on the Massart noise model.
Here we summarize the most relevant literature on learning halfspaces in
related noise models.

\paragraph{Random Classification Noise}

Random Classification Noise (RCN)~\cite{AL88} is the special case of Massart
noise where each label is flipped with probability {\em exactly} $\eta <1/2$.
Halfspaces are known to be efficiently learnable {\em to optimal accuracy}
in the (distribution-independent) PAC model with RCN~\cite{BlumFKV96, BFK+:97}.
In fact, it is well-known that any SQ learning algorithm~\cite{Kearns:98} can be transformed to an
RCN noise tolerant learning algorithm --- a fact that inherently fails in the presence of Massart noise.
Roughly speaking, the ability of the Massart adversary to choose {\em whether} to flip a given label
and, if so, with what probability, makes the algorithmic problem in this model 
significantly more challenging.

\paragraph{Agnostic Learning}
The agnostic model~\cite{Haussler:92, KSS:94} is the strongest noise model in the literature, 
where an adversary is allowed to adversarially corrupt
an arbitrary $\opt<1/2$ fraction of the labels. In the distribution-independent setting,
even {\em weak} agnostic PAC learning of halfspaces (i.e., obtaining a hypothesis with non-trivial accuracy)
is known to be intractable. A long line of work (see, e.g.,~\cite{GR:06, FGK+:06short})
has established NP-hardness of weak agnostic {\em proper} learning.
(See~\cite{Feldman15} for a survey on hardness of proper learning results.)
More recently, \cite{Daniely16} gave super-polynomial lower bounds
for {\em improper} learning, under certain average-case complexity assumptions,
and simultaneously established SQ lower bounds for the problem.
Concretely, \cite{Daniely16} showed that no polynomial-time SQ algorithm for agnostically
learning halfspaces on $\R^m$ can compute a hypothesis with error $1/2-1/m^c$, for some constant $c>0$,
even for instances with optimal error $\opt = 2^{-\log^{1-\nu}(m)}$, for some constant $\nu \in (0, 1/2)$.

Finally, it is worth noting that learning to {\em optimal} accuracy in the agnostic model is known to be computationally 
hard even in the distribution-specific PAC model, and in particular under the Gaussian distribution~\cite{KlivansK14, GGK20-agnostic-SQ, DKZ20, DKPZ21}. However, these distribution-specific hardness 
results are very fragile and do not preclude efficient constant factor approximations. 
In fact, efficient constant factor approximate learners are known for the Gaussian 
and other well-behaved distributions (see, e.g.,~\cite{ABL17, DKS18-nasty}).

\paragraph{Prior SQ Lower Bound for Massart Halfspaces}
\cite{CKMY20} showed an SQ lower bound of $m^{\Omega(\log(1/\eps))}$
for learning halfspaces with Massart to error $\opt+\eps$, when $\opt$ is close to $1/2$.
Specifically,~\cite{CKMY20} observed a connection between SQ learning with Massart noise and the
{\em Correlational Statistical Query (CSQ)} model, a restriction of the SQ model
defined in~\cite{BshoutyFeldman:02} (see also~\cite{Feldman08, Feldman11}).
Given this observation, \cite{CKMY20} deduced their SQ lower bound 
by applying {\em as a black-box} a previously known CSQ lower bound 
by Feldman~\cite{Feldman11}. This approach is inherently limited to exact learning.
Establishing lower bounds for approximate learning requires new ideas.

\paragraph{Efficient Algorithms for Distribution-Specific Learning}
Finally, we note that $\poly(m, 1/\eps)$ time learning algorithms for homogeneous Massart 
halfspaces with optimal error guarantees have been developed 
when the marginal distribution on examples 
is well-behaved~\cite{AwasthiBHU15, AwasthiBHZ16, ZhangLC17, 
YanZ17, Zhang20, DKTZ20, DKTZ20b, DKKTZ20, DKKTZ21b}.
The hardness result obtained in this paper provides additional motivation
for such distributional assumptions. As follows from our inapproximability result, 
without some niceness assumption on the distribution of examples,
obtaining even extremely weak relative approximations to the optimal error 
is hard.

\paragraph{Broader Context}
This work is part of the broader direction of understanding the computational complexity of
robust high-dimensional learning in the distribution-independent setting.
A long line of work, see, e.g.,~\cite{KLS09, ABL17, DKKLMS16, LaiRV16, DKK+17, DKKLMS18-soda,
DKS18-nasty, KlivansKM18, DKS19, DKK+19-sever} and the recent survey~\cite{DK19-survey},
has given efficient robust learners for a range of high-dimensional estimation tasks (both supervised and unsupervised)
in the presence of a small constant fraction of adversarial corruptions.
These algorithmic results inherently rely on the assumption that the clean data
is drawn from a ``well-behaved'' distribution.

On the other hand, the recent work~\cite{DGT19} established that efficient robust learners with non-trivial error guarantees
are achievable {\em even in the distribution-independent setting}, under the more ``benign'' Massart model.
This result provided compelling evidence that there are realistic noise models in which efficient algorithms
are possible without imposing assumptions on the good data distribution.
Conceptually, the result of this paper shows that, even in such semi-random noise models,
there can be strong computational limitations in learnability --- in the sense that
it is computationally hard to achieve even {\em weak relative approximations} to the optimal error.

\subsection{Overview of Techniques} \label{ssec:techniques}

At a high level, our proof leverages the SQ lower bound framework developed in~\cite{DKS17-sq}.
We stress that, while this framework is a key ingredient of our construction, employing it 
in our context requires new conceptual and
technical ideas, as we explain in the proceeding discussion.

Roughly speaking, the prior work \cite{DKS17-sq} established the following generic SQ-hardness result:
Let $A$ be a one-dimensional distribution that matches the first $k$ moments with the standard Gaussian $G$
\new{and satisfying the additional technical condition that its chi-squared norm with $G$ is not too large.}
Suppose we want to distinguish between the standard high-dimensional Gaussian $N(0, I)$ on $\R^m$ and a
distribution $\p^A_v$ that is a copy of $A$ in a random direction \new{$v$} 
and is a standard Gaussian in the orthogonal complement.
Then any SQ algorithm for this hypothesis testing task requires super-polynomial complexity. 
\new{Roughly speaking, any SQ algorithm distinguishing between the two cases requires either
at least $m^{\Omega(k)}$ samples or at least $2^{m^{\Omega(1)}}$ time.}

In the context of the current paper, we will in fact require a generalization of the latter generic result that holds
even if the one-dimensional distribution $A$ {\em nearly} matches the first $k$ moments with $G$.
Furthermore, in contrast to the unsupervised estimation problem studied in \cite{DKS17-sq},
in our context we require a generic statement establishing the SQ-hardness of a {\em binary classification} problem.
Such a statement (Proposition~\ref{prop:generic-sq}) is not hard to derive from the techniques of \cite{DKS17-sq}.

In more detail, Proposition~\ref{prop:generic-sq} shows the following: 
Let $A$ and $B$ be univariate distributions (approximately) matching their first $k$ moments with $G$ 
(and each having not too large chi-squared \new{norm with respect to} $G$)
and let $p\in (0,1)$. We consider the distribution on labeled samples $\p^{A,B,p}_v$
that returns a sample from $(\p^A_v,1)$ with probability $p$ 
and a sample from $(\p^B_v,-1)$ with probability $1-p$.
Given labeled examples from $\p^{A,B,p}_v$, for an unknown direction $v$,
the goal is to output a Boolean-valued hypothesis with small misclassification error.
Note that it is straightforward to obtain error $\min\{p, 1-p \}$
(as one of the two constant functions achieves this). 
We show that obtaining slightly better error is hard in the SQ model.

To leverage the aforementioned result in our circumstances,
we would like to establish the existence of a distribution $(X,Y)$ on $\R \times \{ \pm 1\}$
{\em that corresponds to a halfspace with Massart noise} 
such that both the distribution of $X $ conditioned on $Y=1$ (denoted by $(X \mid Y=1)$)
and the distribution of $X$ conditioned on $Y=-1$ (denoted by $(X \mid Y=-1)$) 
approximately match their first $k$ moments with the standard Gaussian.
{\em Note that $k$ here is a parameter that we would like to make as large as possible. In particular,
to prove a super-polynomial SQ lower bound, we need to be able to make this parameter 
$k$ {\em super-constant} (as a function of the ambient dimension).}

Naturally, a number of obstacles arise while trying to achieve this.
{\em In particular, achieving the above goal directly is provably impossible} for the following reason.
Any distribution $X$ that even approximately matches a {\em constant} number
of low-order moments with the standard Gaussian will satisfy $\E[f(X)] \approx \E[f(G)]$ 
for any halfspace (LTF) $f$.
To see this fact, we can use the known statement (see, e.g.,~\cite{DGJ+:10})
that any halfspace $f$ can be sandwiched between low-degree polynomials 
$f_+\geq f \geq f_-$ with $\E[f_+(G)-f_-(G)]$ small.
This structural result implies that if both conditional distributions $(X \mid Y=1)$ and $(X \mid Y=-1)$
approximately match their low-degree moments with $G$,
then $\E[f(X)|Y=1]$ will necessarily be close to $\E[f(X)|Y=-1]$, 
which cannot hold in the presence of Massart noise.

In order to circumvent this obstacle, we will instead prove a super-polynomial SQ lower bound
against learning degree-$d$ polynomial threshold functions (PTFs) under the Gaussian distribution
with Massart noise, for an appropriate (super-constant) value of the degree $d$.
Since a degree-$d$ PTF on the vector random variable $X \in \R^m$ is equivalent 
to an LTF on $X^{\otimes d}$ --- a random variable in $m^d$ dimensions ---
we will thus obtain an SQ lower bound for the original halfspace Massart learning problem.
We note that a similar idea was used in~\cite{Daniely16} to prove an SQ lower bound 
for the problem of learning halfspaces in the agnostic model.

The next challenge is, of course, to construct the required moment-matching distributions in one dimension.
{\em Even for our reformulated PTF learning problem, it remains unclear whether this is even possible.}
For example, let $f(x)=\sgn(p(x))$ be a degree-$d$ PTF.
Then it will be the case that $\E[p(X)Y] = \E[p(X)f(X)(1-2\eta(X))] = \E[|p(X)|(1-2\eta(X))] >0$.
This holds despite the fact that $\E[p(X) \mid Y=1] \approx \E[p(X) \mid Y=-1] \approx \E[p(G)]$.
If $\E[p(G)]>0$, it will be the case that $\E[p(X) \mid Y=-1]$ will be positive,
despite the fact that the conditional distribution of $X \mid Y=-1$ is almost entirely supported
on the region where $p(X) < 0$. Our construction will thus need to take advantage of finding points
where $|p(X)|$ is very large.


Fortunately for us, something of a miracle occurs here. Consider a discrete univariate Gaussian $G_{\delta}$ 
with spacing \new{$\sigma$} between its values. It is not hard to show that $G_{\delta}$ approximately matches moments
with the standard Gaussian $G$ to error $\exp(-\Omega(1/\new{\sigma}^2))$ (see Lemma~\ref{lem:disc-Gaussian-mm}). 
On the other hand, all but a tiny fraction of the probability mass of $G_{\new{\sigma}}$ is supported 
on $d=\tilde O(1/\new{\sigma})$ points. 
Unfortunately, a discrete Gaussian is not quite suitable for the conditional distributions in our construction, 
as its $\chi^2$ inner product with respect to the standard Gaussian is infinite. 
We can fix this issue by replacing the single discrete Gaussian with an average of discrete Gaussians with different offsets. 
Doing so, we obtain a distribution that nearly-matches many moments with the standard Gaussian
such that all but a small fraction of its mass is supported on a small number of intervals.

As a first attempt, we can let one of our conditional distributions be 
this average of ``offset discrete Gaussians'' described above, 
and the other be a similar average with different offsets. Thus, both conditional distributions nearly-match moments 
with the standard Gaussian and are approximately supported on a small number of (disjoint) intervals. 
This construction actually suffices to prove a lower bound for (the much more challenging) agnostic learning model. 
Unfortunately however, for the Massart noise model, additional properties are needed. 
In particular, for a univariate PTF with Massart noise, it must be the case that except for points $x$ 
in a small number of intervals, we have that $\pr[Y=1 \mid X=x] > \pr[Y=-1 \mid X=x]$; 
whereas in the above described construction 
we have to alternate infinitely many times between $Y=1$ being more likely and $Y=-1$ more likely.

To circumvent this issue, we need the following subtle modification of our construction. 
Let $G_{\new{\sigma},\theta}$ be the discrete Gaussian supported on the points $n\new{\sigma}+\theta$, 
for $n\in \Z$ (Definition~\ref{def:disc-Gauss}). 
Our \new{previous (failed)} construction involved taking an average of 
$G_{\new{\sigma},\theta}$, for some \emph{fixed} $\new{\sigma}$ and $\theta$ varying in some range. 
Our modified construction will involve taking an average of $G_{\new{\sigma},\theta}$, 
where both $\new{\sigma}$ and $\theta$ {\em vary together}. The effect of this feature will be that 
instead of producing a distribution whose support is a set of evenly spaced intervals of the same size, 
the support of our distributions will instead consist of a set of evenly spaced intervals 
whose size grows with the distance from the origin. This means that for points $x$ near $0$, 
the support will essentially be a collection of small, disjoint intervals. 
But when $x$ becomes large enough, these intervals will begin to overlap, 
causing all sufficiently large points $x$ to be in our support. 
By changing the offsets used in defining the conditional distribution for $Y=1$ 
and the conditional distribution for $Y=-1$, we can ensure that for points $x$ with $|x|$ small 
that the supports of the two conditional distributions remain disjoint. This in particular allows us 
to take the optimal error $\opt$ to be very small. 
However, for larger values of $|x|$, the supports become the same. 
Finally, by adjusting the prior probabilities of $Y=1$ and $Y=-1$, 
we can ensure that $\pr[Y=1 \mid X=x] > ((1-\eta)/(\eta)) \, \pr[Y=-1 \mid X=x]$ 
for all points $x$ with $|x|$ sufficiently large. This suffices to show that the underlying
distribution corresponds to a Massart PTF.

\new{
\subsection{Organization} \label{ssec:structure}
The structure of this paper is as follows: 
In Section~\ref{sec:prelims}, we review the necessary background on the Statistical Query model. In Section~\ref{sec:proof-cont}, we prove our SQ lower bounds for Massart halfspaces. Finally, in Section~\ref{sec:conc}, we conclude and suggest a few 
directions for future work.
}

\section{Preliminaries} \label{sec:prelims}

\paragraph{Notation} 
For $n \in \Z_+$, we denote $[n] \eqdef \{1, \ldots, n\}$.
We will use standard notation for norms of vectors and functions,
that will be presented before it is used in subsequent sections.
We use $\E[X]$ for the expectation of random variable $X$ and $\pr[\mathcal{E}]$
for the probability of event $\mathcal{E}$.


\paragraph{Basics on Statistical Query Algorithms.} 
We will use the framework of Statistical Query (SQ) algorithms for problems over distributions
introduced in~\cite{FGR+13}.
We start by defining a decision problem over distributions. 

\begin{definition}[Decision Problem over Distributions] \label{def:decision}
We denote by $\mathcal{B}(\D, D)$ the decision (or hypothesis testing) problem in which 
the input distribution $D'$ is promised to satisfy either (a) $D' = D$ or (b) $D' \in \D$, 
and the goal of the algorithm is to distinguish between these two cases.
\end{definition}

We define SQ algorithms as algorithms that do not have direct access to samples from the distribution, 
but instead have access to an SQ oracle. We consider the following standard oracle.

\begin{definition}[$\mathrm{STAT}$ Oracle]\label{def:stat}
For a tolerance parameter $\tau >0$ and any bounded function $f: \R^n \to [-1, 1]$, 
$\mathrm{STAT}(\tau)$ returns a value $v \in \left[\E_{x \sim D} [f(x)] - \tau,  \E_{x \sim D} [f(x)] + \tau \right]$.
\end{definition}

We note that \cite{FGR+13} introduced another related oracle, 
which is polynomially equivalent to $\mathrm{STAT}$. 
Since we prove super-polynomial lower bounds here, there is no essential distinction between these oracles.
To define the SQ dimension, we need 
the following definitions.

\begin{definition}[Pairwise Correlation] \label{def:pc}
The pairwise correlation of two distributions with probability density functions 
$D_1, D_2 : \R^m \to \R_+$ with respect to a distribution with density $D: \R^m \to \R_+$, 
where the support of $D$ contains the supports of $D_1$ and $D_2$, 
is defined as $\chi_{D}(D_1, D_2) \eqdef \int_{\R^m} D_1(x) D_2(x)/D(x) dx -1$.
\end{definition} 

We remark that when $D_1=D_2$ in the above definition, 
the pairwise correlation is identified with the $\chi^2$-divergence between $D_1$ and $D$, 
i.e.,  $\chi^2(D_1, D) \eqdef \int_{\R^m} D_1(x)^2/D(x) dx -1$. 

\begin{definition} \label{def:uncor}
We say that a set of $s$ distributions $\mathcal{D} = \{D_1, \ldots , D_s \}$ over $\R^m$ 
is $(\gamma, \beta)$-correlated relative to a distribution $D$ if 
$|\chi_D(D_i, D_j)| \leq \gamma$ for all $i \neq j$, and $|\chi_D(D_i, D_j)| \leq \beta$ for $i=j$.
\end{definition}

We are now ready to define our notion of dimension.

\begin{definition}[Statistical Query Dimension] \label{def:sq-dim}
For $\beta, \gamma > 0$, a decision problem $\mathcal{B}(\D, D)$,
where $D$ is a fixed distribution and $\D$ is a family of distributions over $\R^m$, 
let $s$ be the maximum integer such that there exists a finite set of distributions 
$\mathcal{D}_D \subseteq \D$ 
such that $\mathcal{D}_D$ is $(\gamma, \beta)$-correlated relative to $D$ and $|\mathcal{D}_D| \geq s.$ 
We define the {\em Statistical Query dimension} with pairwise correlations $(\gamma, \beta)$ 
of $\mathcal{B}$ to be $s$ and denote it by $\mathrm{SD}(\mathcal{B},\gamma,\beta)$.
\end{definition}

Our proof bounds below the Statistical Query dimension
of the considered learning problem. This implies lower bounds on 
the complexity of any SQ algorithm for the problem using the following standard result.

\begin{lemma}[Corollary 3.12 in~\cite{FGR+13}] \label{lem:sq-from-pairwise} 
Let $\mathcal{B}(\D, D)$ be a decision problem, where $D$ is the reference distribution
and $\mathcal{D}$ is a class of distributions. For $\gamma, \beta >0$, 
let $s= \mathrm{SD}(\mathcal{B}, \gamma, \beta)$. For any $\gamma' > 0,$ any
SQ algorithm for $\mathcal{B}$ requires at least $s \cdot \gamma' /(\beta - \gamma)$ queries to the 
$\mathrm{STAT}(\sqrt{\gamma + \gamma'})$ oracle.
\end{lemma}

\section{SQ Hardness of Learning Halfspaces with Massart Noise} \label{sec:proof-cont}
In this section, we prove our SQ lower bounds for Massart halfspaces, 
establishing Theorem~\ref{thm:main-inf}. In more detail, we establish the following
result, which implies Theorem~\ref{thm:main-inf}.

\begin{theorem}[SQ Hardness of Learning Massart Halfspaces on $\R^M$] \label{thm:main-formal}
Let $\opt>0$ and $M \in \Z_+$ be such that
$\log(M)/(\log\log(M))^3$ is at least a sufficiently large constant multiple of $\log(1/\opt)$.
There exists a parameter $\tau \eqdef M^{-\Omega\left( \frac{\log(M)}{{\log\log(M)}^3}/\log(1/\opt) \right)}$
such that no SQ algorithm can learn the class of halfspaces on $\R^M$
in the presence of $\eta$-Massart noise, where $\opt< \eta <1/2$, within
error better than $\Omega(\eta)$ using at most $1/\tau$ queries
of tolerance $\tau$.
This holds even if the optimal binary classifier has misclassification error at most $\opt$.
\end{theorem}

\new{As an immediate corollary of Theorem~\ref{thm:main-formal}, 
by taking $\opt = \exp(- \log^{c'}(M))$ for any fixed
constant $c' \in (0, 1)$, we obtain a super-polynomial SQ lower bound
against learning a hypothesis with error better than $\Omega(\eta)$,
even when error $\opt$ is possible. Specifically, this setting of parameters immediately implies 
Theorem~\ref{thm:main-inf}, since 
$\frac{\log(M)}{{\log\log(M)}^3}\new{/}\log(1/\opt) = \frac{\log^{1-c'}(M)}{{\log\log(M)}^3} > \log^{c}(M)$, 
where $0< c < 1-c'$, and therefore $1/\tau  \gg \exp(\log^{1+c}(M))$.

\begin{remark}\label{rem:sharp-error}
{\em In addition to Theorem~\ref{thm:main-formal}, we establish 
an alternative ``inapproximability gap'' of $1/2- O(\sqrt{1/2-\eta})$ versus
$\exp(- \log^{c'}(M))$, which implies a sharper error lower bound when $\eta$ approaches $1/2$.
Specifically, for $\eta$ close to $1/2$, we obtain an error lower bound of $\eta - o_{\eta}(1)$, even
if $\opt = \exp(- \log^{c'}(M))$. See Theorem~\ref{thm:main-large-eta} for the formal statement.}
\end{remark}

}

\paragraph{Structure of This Section}
The structure of this section is as follows:
In Section~\ref{ssec:high-dim}, we review the SQ framework of~\cite{DKS17-sq} 
with the necessary enhancements and modifications required for our supervised setting.
In Section~\ref{ssec:one-dim}, we establish the existence of the one-dimensional distributions
with the desired approximate moment-matching properties.
In Section~\ref{ssec:combine}, we put everything together
to complete the proof of Theorem~\ref{thm:main-formal}.
In Section~\ref{ssec:large-eta}, we establish our shaper lower bounds for $\eta$ close to $1/2$,
proving Theorem~\ref{thm:main-large-eta}.

\subsection{Generic SQ Lower Bound Construction} \label{ssec:high-dim}

We start with the following definition:

\begin{definition} [High-Dimensional Hidden Direction Distribution] \label{def:pv-hidden}
{For a distribution $A$ on the real line with probability density function $A(x)$ and}
 a unit vector $v \in \R^m$, consider the distribution over $\R^m$ with probability density function
$\p^A_v(x) = A(v \cdot x) \exp\left(-\|x - (v \cdot x) v\|_2^2/2\right)/(2\pi)^{(m-1)/2}.$
That is, $\p^A_v$ is the product distribution whose orthogonal projection onto the direction of $v$ is $A$,
and onto the subspace perpendicular to $v$ is the standard $(m-1)$-dimensional normal distribution.
\end{definition}

We consider the following condition:

\begin{cond} \label{cond:moments}
Let $k \in \Z_+$ and $\nu>0$.
The distribution $A$ is such that (i) the first $k$ moments of $A$
agree with the first $k$ moments of $N(0,1)$ up to error at most $\nu$, and (ii) $\chi^2(A,N(0,1))$ is finite.
\end{cond}

Note that Condition~\ref{cond:moments}-(ii) above
implies that the distribution $A$ has a pdf, which we will denote by $A(x)$.
We will henceforth blur the distinction between a distribution and its pdf.

Our main result in this subsection makes essential use of the following key lemma:

\begin{lemma}[Correlation Lemma] \label{lem:cor}
Let $k \in \Z_+$.
If the univariate distribution $A$ satisfies Condition \ref{cond:moments},
then for all $v,v' \in \R^m$, with $|v\cdot v'|$ less than a sufficiently small constant, we have that
\begin{equation} \label{eqn:corr-pv}
|\chi_{N(0,I)}(\p^A_v, \p^A_{v'})| \leq |v \cdot v'|^{k+1} \chi^2(A, N(0,1)) + \nu^2 \;.
\end{equation}
\end{lemma}
This lemma is a technical generalization of Lemma~3.4 from~\cite{DKS17-sq},
which applied under {\em exact} moment matching assumptions. The proof is deferred to Appendix~\ref{app:cor}.

We will also use the following standard fact:

\begin{fact}\label{fact:near-orth-vec}
For any constant $c>0$ there exists a set $S$ of $2^{\Omega_c(m)}$ unit vectors in $\R^m$
such that any pair $u, v \in S$, with $u \neq v$, satisfies $|u\cdot v|<c$.
\end{fact}

In fact, an {appropriate size} set of random unit vectors
satisfies the above statement {with high probability}.
We note that \cite{DKS17-sq} made use of a similar statement, albeit with different parameters.

We will establish an SQ lower bound for the following binary classification problem.

\begin{definition}[Hidden Direction Binary Classification Problem] \label{def:class}
Let $A$ and $B$ be distributions on $\R$ satisfying Condition \ref{cond:moments} with
parameters $k \in \Z_+$ and $\nu \in \R_+$, and let $p\in (0,1)$. For $m \in \Z_+$ and a unit vector $v\in \R^m$,
define the distribution $\p^{A,B,p}_v$ on $\R^m\times \{\pm 1\}$ that returns a sample
from $(\p^A_v,1)$ with probability $p$ and a sample from $(\p^B_v,-1)$ with probability $1-p$.
The corresponding binary classification problem is the following: Given access to a distribution on labeled
examples of the form $\p^{A,B,p}_v$, for a fixed but unknown unit vector $v$,
output a hypothesis $h: \R^m \to \{\pm1\}$ such that $\pr_{(X,Y) \sim \p^{A,B,p}_v}[h(X)\neq Y]$
is (approximately) minimized.
\end{definition}

Note that it is straightforward to obtain misclassification error $\min\{p, 1-p \}$
(as one of the identically constant functions achieves this guarantee).
We show that obtaining slightly better error is hard in the SQ model.
The following result is the basis for our SQ lower bounds:

\begin{proposition}[Generic SQ Lower Bound] \label{prop:generic-sq}
{Consider the classification problem of Definition~\ref{def:class}.}
Let $\tau \eqdef \nu^2 + 2^{-k}(\chi^2(A, N(0,1))+\chi^2(B, N(0,1)))$.
Then any SQ algorithm that, {given access to a distribution $\p^{A,B,p}_v$ for an unknown
$v \in \R^m$,} outputs a hypothesis $h: \R^m \to \{\pm1\}$ such that
$\pr_{(X,Y) \sim \p^{A,B,p}_v}[h(X)\neq Y] < \min(p,1-p)- \new{4} \sqrt{\tau}$ must either make queries
of accuracy better than $2\sqrt{\tau}$ or must make at least
$2^{\Omega(m)}\tau/(\chi^2(A, N(0,1))+\chi^2(B, N(0,1)))$ statistical queries.
\end{proposition}

\begin{proof}[Proof of Proposition~\ref{prop:generic-sq}]
The proof proceeds as follows: We start by defining a related hypothesis testing problem $\mathcal{H}$
and show that $\mathcal{H}$ efficiently reduces to our learning (search) problem. We then leverage the machinery
of this section (specifically, Lemma~\ref{lem:cor} and Fact~\ref{fact:near-orth-vec}) to prove an SQ lower bound
for $\mathcal{H}$, which in turns implies an SQ lower bound for our learning task.

Let $S$ be a set of $2^{\Omega(m)}$ unit vectors in $\R^m$ whose pairwise inner products are
at most a sufficiently small universal constant $c$. {(In fact, any constant $c<1/2$ suffices.)}
By Fact~\ref{fact:near-orth-vec}, such a set is guaranteed to exist.
Given $S$, our hypothesis testing problem is defined as follows.

\begin{definition}[Hidden Direction Hypothesis Testing Problem] \label{def:hyp-test}
In the context of Definition~\ref{def:class}, the testing problem $\mathcal{H}$
is the task of distinguishing between: (i) the distribution $\p^{A,B,p}_v$, for $v$ randomly chosen from $S$,
and (ii) the distribution $G'$ on $\R^m \times \{ \pm 1\}$, where for $(X, Y) \sim G'$ we have that
$X$ is a standard Gaussian {$G \sim N(0, I)$},
and $Y$ is independently $1$ with probability $p$ and $-1$ with probability $1-p$.
\end{definition}

We claim that $\mathcal{H}$ efficiently reduces to our learning task. In more detail,
any SQ algorithm that computes a hypothesis $h$ satisfying
$\pr_{(X,Y) \sim \p^{A,B,p}_v}[h(X)\neq Y] < \min(p,1-p)-\new{4} \sqrt{\tau}$ can be used as a black-box
to distinguish between $\p^{A,B,p}_v$, for $v$ randomly chosen from $S$, and $G'$.
Indeed, suppose we have such a hypothesis $h$.
Then, with one additional {query} to estimate the $\pr[h(X)\neq Y]$,
we can distinguish between $\p^{A,B,p}_v$, for $v$ randomly chosen from $S$,
and $G'$ for the following reason: For any function $h$, we have that
$\pr_{(X,Y)\sim G'}[h(X)\neq Y] \geq \min(p,1-p)$.

It remains to prove that solving the hypothesis testing problem $\mathcal{H}$
is impossible for an SQ algorithm with the desired parameters.
We will show this using {Lemma~\ref{lem:sq-from-pairwise}}.

More specifically, we need to show that for $u , v \in S$ we have
that $|\chi_{G'}(\p^{A,B,p}_v,\p^{A,B,p}_u)|$ is small.
Since $G',\p^{A,B,p}_v,$ and $\p^{A,B,p}_u$ all assign $Y=1$ with probability $p$,
it is not hard to see that
\begin{align*}
\chi_{G'}(\p^{A,B,p}_v,\p^{A,B,p}_u)
= & \; p \; \chi_{(G' \mid Y=1)}\left( (\p^{A,B,p}_v \mid Y=1) , (\p^{A,B,p}_u \mid Y=1) \right) + \\
& (1-p) \; \chi_{(G' \mid Y=-1)} \left( (\p^{A,B,p}_v \mid Y=-1) , (\p^{A,B,p}_u \mid Y=-1) \right)\\
 = & \; p \; \chi_{G}(\p^A_v , \p^A_u) + (1-p) \; \chi_{G}(\p^B_v, \p^B_u) \;.
\end{align*}
By Lemma \ref{lem:cor}, it follows that
$$\chi_{G'}(\p^{A,B,p}_v,\p^{A,B,p}_u) \leq \nu^2 + 2^{-k}(\chi^2(A, N(0,1))+\chi^2(B, N(0,1)))=\tau \;.$$
A similar computation shows that
$$\chi_{G'}(\p^{A,B,p}_v,\p^{A,B,p}_v) = \chi^2(\p^{A,B,p}_v,G')\leq \chi^2(A, N(0,1))+\chi^2(B, N(0,1)) \;.$$
An application of Lemma~\ref{lem:sq-from-pairwise} for $\gamma = \gamma' = \tau$ and
$\beta = \chi^2(A, N(0,1))+\chi^2(B, N(0,1))$ completes the proof.
\end{proof}

\subsection{Construction of Univariate Moment-Matching Distributions} \label{ssec:one-dim}
\new{In this subsection, we give our univariate approximate moment-matching construction (Proposition~\ref{mainProp}),
which is the key new ingredient to establish our desired SQ lower bound. 
The moment-matching construction of this subsection 
(along with its refinement for $\eta$ close to $1/2$ presented
in Section~\ref{ssec:large-eta}) is the main technical contribution of this work.}

\paragraph{Additional Notation}
We will be working with moments of distributions
that are best described as normalizations of (unnormalized positive) measures.
For notational convenience, by slight abuse of notation,
we define $\E[X]$, for any non-negative measure on $\R$, by $\E[X] = \int tdX(t).$
Note that this is equivalent to $\E[X] = \|X\|_1 \E[Y]$, where $Y=X/\|X\|_1$
is the normalized version of $X$.
Furthermore, we denote the $k^{th}$ moment of such an $X$ by $\E[X^k]$.

We write $E \gg F$ for two expressions $E$ and $F$ to denote that $E \geq c \, F$, where $c>0$
is a sufficiently large universal constant (independent of the variables or parameters on which $E$ and $F$ depend).
Similarly, we write $E \ll F$ to denote that $E \leq c \, F$, where $c>0$
is a sufficiently small universal constant.


We will use $G$ for the measure of the univariate standard Gaussian distribution $N(0, 1)$
and $g(x) = \frac{1}{\sqrt{2\pi}} \exp(-x^2/2)$ for its probability density function.

\medskip

We view the real line as a measurable space endowed with the $\sigma$-algebra
of Lebesgue measurable sets. We will construct two (non-negative, finite) measures $\D_{+}$ and $\D_{-}$
on this space with appropriate properties.
The main technical result of this section is captured in the following proposition.

\begin{proposition}\label{mainProp}
Let $0< \eps < s <1$ be real numbers such that $s/\eps$ is at least a sufficiently large universal constant. 
Let \new{$0< \eta < 1/2$}.
There exist measures $\D_{+}$ and $\D_{-}$ over $\R$ and a union $J$ 
of $d = O(s/\eps)$ intervals on $\R$ such that:
\begin{enumerate}
\item \label{prop:1} (a) $\D_{+}=0$ on $J$, and (b) $\D_{+} / \D_{-} > (1-\eta)/\eta$ on $J^c := \R \setminus J$.
\item \label{prop:2} All but $\zeta = O(\eta s/\eps)\exp(-\Omega(s^4/\eps^2))$ of the measure of $\D_{-}$ lies in $J$.
\item \label{prop:3} For any $t \in \N$, the distributions $\D_{+}/\|\D_{+}\|_1$ and $\D_{-}/\|\D_{-}\|_1$ have their
first $t$ moments matching those of $G$ within additive error at most ${(t+1)!} \exp(-\Omega(1/s^2))$.
\item \label{prop:4} (a) $\D_{+}$ is at most $O(s/\eps) \; G$, and
(b) $\D_{-}$ is at most $O(s\eta /\eps) \; G$.
\item \label{prop:5} {(a)} $\|\D_+\|_1 = \Theta(1)$, {and (b) $\|\D_-\|_1 = \Theta(\eta)$}.
\end{enumerate}
\end{proposition}

\paragraph{Discussion}
Essentially, in our final construction, $\D_{+}$ will be proportional to the distribution 
of $X$ conditioned on $Y=1$ and $\D_{-}$ proportional to the distribution of $X$ conditioned on $Y=-1$.
Furthermore, the ratio of the probability of $Y=1$ to the probability of $Y=-1$ will be equal to $\|\D_{+}\|_1/ \|\D_{-}\|_1$. 
\new{The Massart PTF that $f(X)$ is supposed to simulate will be $-1$ on $X \in J$ and $1$ elsewhere 
(thus making it a degree-$2d$ PTF).}

We now provide an explanation of the properties established in Proposition~\ref{mainProp}.
Property \ref{prop:1}(a) says that $Y$ will deterministically be $-1$ on $J$, while property
\ref{prop:1}(b) says that the ratio between $\D_{+}$ and $\D_{-}$ will be greater than 
$(1-\eta)/\eta$ on the complement of $J$. \new{This implies that $Y$ amounts to $f(X)$ 
with Massart noise at most $\eta$.}

Property \ref{prop:2} implies that $Y$ only disagrees with the target PTF with probability roughly $\zeta$, 
i.e., that the optimal misclassification value $\opt$ will be less than $\zeta$. 

Property \ref{prop:3} says that $\D_{+}$ and $\D_{-}$, after rescaling, approximately match
many moments with the standard Gaussian, which will be necessary in establishing our SQ lower bounds.

\new{
Property \ref{prop:4}  is necessary to show that $\D_{+}$ and $\D_{-}$ 
have relatively small chi-squared norms.
Finally, Property\ref{prop:5} is necessary to  figure out how big the parameter $p$ (i.e., $\Pr[Y=1]$) 
should be (approximately).
}


\medskip

The rest of this subsection is devoted to the proof of Proposition~\ref{mainProp}.

\begin{proof}[Proof of Proposition~\ref{mainProp}]
The proof will make essential use of the following two-parameter family of discrete Gaussians.

\begin{definition}[Discrete Gaussian] \label{def:disc-Gauss}
For $\sigma \in \R_+$ and $\theta \in \R$, let $G_{\sigma,\theta}$
denote the measure of the ``$\sigma$-spaced discrete Gaussian distribution''.
In particular,  for each $n\in \Z$, $G_{\sigma,\theta}$ assigns mass $\sigma g(n\sigma+\theta)$ to the point $n \sigma+\theta$.
\end{definition}

Note that $G_{\sigma,\theta}$ is not a probability measure as its total measure is not equal to one. 
However, it is not difficult to show (see Lemma~\ref{lem:disc-Gaussian-mm} below) 
that the measure of $G_{\sigma,\theta}$ is close to one for small $\sigma>0$, 
hence can be intuitively thought of as a probability distribution.

The following lemma shows that the moments of $G_{\sigma,\theta}$
approximately match the moments of the standard Gaussian measure $G$.

\begin{lemma}\label{lem:disc-Gaussian-mm}
For all $t \in \N, \sigma \geq 0$, and all $\theta\in \R$ we have that
$\left| \E[G_{\sigma,\theta}^t] - \E[G^t] \right| = t! \, O(\sigma)^t \, \exp(-\Omega(1/\sigma^2))$.
\end{lemma}

The proof of Lemma~\ref{lem:disc-Gaussian-mm} proceeds by analyzing
the Fourier transform of $G_{\sigma,\theta}$ and using the fact that the $t^{th}$ moment
of a measure is proportional to the $t^{th}$ derivative of its Fourier transform at $0$.
The proof is deferred to Appendix~\ref{app:disc-Gaussian}.

Note that Lemma~\ref{lem:disc-Gaussian-mm} for $t=0$ implies
the total measure of $G_{\sigma,\theta}$ is $\exp(-\Omega(1/\sigma^2))$ close to one, 
i.e., for small $\sigma>0$ $G_{\sigma,\theta}$ can be thought of as a probability distribution.

\paragraph{Definition of the Measures $\D_+$ and $\D_-$}

We define our measures as mixtures of discrete Gaussian distributions. 
This will allow us to guarantee that they nearly match moments with the standard Gaussian. 
In particular, for a  sufficiently large constant $C>0$, we define:
\begin{equation} \label{eq:d-plus}
\D_+ := C \, (s/\eps) \, \int_{0}^{\eps} \frac{1}{s+y} \, G_{s+y,y/2} \, dy \;,
\end{equation}
and
\begin{equation} \label{eq:d-minus}
\D_- := \eta \, (s/\eps) \, \int_{0}^{\eps} \frac{1}{s+y} \, G_{s+y,(y+s)/2} dy \;.
\end{equation}

We will require the following explicit formulas for $\D_+$ and $\D_{-}$,
which will be useful both in the formal proof and for the sake of the intuition. 

\begin{lemma} \label{lem:formulas-d}
For all $x \in \R$, we have that:
\begin{equation} \label{eq:d-plus-exp}
\D_{+} (x)  = C \, g(x) \, (s/\eps) \, \sum_{n\in \Z}\frac{\mathbf{1}\{ x \in [ns,ns+(n+1/2)\eps]\}}{|n+1/2|} \;,
\end{equation}
where by $\mathbf{1}\{ x \in [ns,ns+(n+1/2)\eps]\}$ we denote the indicator function of the event that $x$ 
is between $ns$ and $ns+(n+1/2)\eps$, even in the case where $n<0$ and $ns+(n+1/2)\eps < ns$.

Similarly, we have that
\begin{equation} \label{eq:d-minus-exp}
\D_{-} (x)  = \eta \, g(x) \,  (s/\eps) \, \sum_{n\in \Z}\frac{\mathbf{1}\{ x \in [(n+1/2)s,(n+1/2)s+(n+1/2)\eps]\}}{|n+1/2|} \;.
\end{equation}
\end{lemma}

\begin{proof}
To prove the lemma, we unravel the definition of the discrete Gaussian to find that:
\begin{align*}
\D_+(x) & = Cs/\eps \int_{0}^\eps \frac{\sum_{n\in \Z} (s+y)g(n(s+y)+y/2)\delta(x-(n(s+y)+y/2)) }{s+y} dy\\
& = Cs/\eps \sum_{n\in \Z}\int_{0}^\eps g(n(s+y)+y/2)\delta(x-(n(s+y)+y/2))  dy\\
& = Cs/\eps \sum_{n\in \Z}\int_{0}^\eps g((n+1/2)y+ns)\delta(x-((n+1/2)y+ns))  dy\\
& = Cg(x)s/\eps \sum_{n\in \Z}\frac{\mathbf{1}\{ x \in [ns,ns+(n+1/2)\eps]\}}{|n+1/2|} \;.
\end{align*}
The calculation for $\D_-(x)$ is similar.
\end{proof}

\paragraph{Intuition on Definition of $\D_+$ and $\D_-$}
We now attempt to provide some intuition regarding the definition
of the above measures. We start by noting that each of $\D_{+} (x)$ and $\D_{-} (x)$ 
will have size roughly $g(x)$ on its support. This can be seen to imply 
on the one hand that the chi-squared divergence of (the normalization of) 
$\D_\pm$ from the standard Gaussian $G$ is not too large, 
and on the other hand that $\D_{\pm}$
roughly satisfy Gaussian concentration bounds. 

The critical information to consider is the support of these distributions. 
Each of the two measures is supported on a union of intervals. 
Specifically, $\D_+$ is supported on intervals located at the point $n \, s$ of width $|n+1/2|\eps$; 
and $\D_-$ is supported on intervals located at the point $(n+1/2)s$ of width $|n+1/2|\eps$. 
In the case where $\eps \ll s$, these intervals will be disjoint for small values of $n$ 
(roughly, for $\new{|n|} \ll s/\eps$). 
The factor of $C>0$ difference in the definitions of the two measures 
will ensure that $\D_+ > \D_- (1-\eta)/\eta$ on their joint support; 
and once $|n|$ has exceeded a sufficiently large constant multiple of $s/\eps$, 
the intervals will be wide enough that they overlap causing the support to be everything.

In other words, for $|x|$ less than a sufficiently small constant multiple of $s^2/\eps$, 
$\D_+$ and $\D_-$ will be supported on $O(s/\eps)$ many intervals and will have disjoint supports. 
We define $J$ to be the union of the $O(s/\eps)$ many intervals in the support of $\D_{-}$ 
that are not in the support of $\D_{+}$. With this definition, we will have that 
(1) $\D_+$ is equal to zero in $J$, and (2) $\D_+/\D_-$ is sufficiently large on $J^c$. 
Furthermore, since $\D_{-}$ only assigns mass to $J^c$ for $x$ with $|x| \gg s^2/\eps$, 
we can take $\zeta = \exp(-\Omega(s^4/\eps^2))$.

\medskip

Given the above intuition, we begin the formal proof, 
starting with moment-matching.

\begin{lemma}\label{DMomentsLem}
For $t \in \N$, the distributions $\D_+ / \|\D_+\|_1$ and $\D_-/\|\D_-\|_1$ 
match the first $t$ moments with the standard Gaussian 
$G$ to within additive error $t! \, O(s)^t \, \exp(-1/s^2)$.
\end{lemma}
\begin{proof}
This follows from Lemma \ref{lem:disc-Gaussian-mm} by noting that 
both of these distributions are mixtures of discrete Gaussians with $\sigma = \Theta(s)$.
\end{proof}

Our next lemma provides approximations to the corresponding $L_1$ norms.

\begin{lemma}\label{DL1Lem}
We have that 
$\|\D_+\|_1 = \Theta(1)$ and $\|D_-\|_1 = \Theta(\eta)$.
\end{lemma}
\begin{proof}
Applying Lemma \ref{lem:disc-Gaussian-mm} with $t=0$, we get that $\|G_{s+y,y/2}\|_1 = \Theta(1)$. 
Thus, working from the definition, we find that
\begin{align*}
\|\D_+\|_1 & = C \, (s/\eps) \, \int_0^{\eps} \frac{\Theta(1)}{s+y} dy\\
& = C \, (s/\eps) \, \int_0^{\eps} \Theta(1/s) dy\\
& = \Theta(1) \;.
\end{align*}
The proof for $\D_-$ follows similarly.
\end{proof}

For the rest of the proof, it will be important to analyze 
the intervals on which $\D_{+}$ and $\D_{-}$ are supported. 
To this end, we start by introducing the following notation.

\begin{definition}\label{def:I-n-plus-minus}
For $m\in \Z$, let $I_{+}^m$ be the interval with endpoints $ms$ and $ms+(m+1/2)\eps$, 
and let $I_{-}^m$ be the interval with endpoints $(m+1/2)s$ and $(m+1/2)s+(m+1/2)\eps$. 
Additionally, for $x \in \R$, let $n_{+}(x)$ be the number of integers $m$ such that $x\in I_{+}^m$, 
and $n_{-}(x)$ be the number of integers $m$ such that $x\in I_{-}^m$. 
\end{definition}

The following corollary is an easy consequence of the definition.

\begin{corollary}\label{mSizeCor}
For $x \in \R$, we have that $x\in I_{+}^m$ only if $m=x/s + O((|x|+s)\eps/s)$. 
Similarly, $x\in I_{-}^m$ only if $m=x/s - 1/2 + O((|x|+s)\eps/s)$.
\end{corollary}

Combining Corollary~\ref{mSizeCor} with the explicit formulas for $\D_+$ and $\D_-$ 
given in Lemma~\ref{lem:formulas-d},
we have that:
\begin{corollary}\label{DSizeCor}
For all $x \in \R$ we have that
$$
\D_+(x) = \Theta(C g(x) (s^2/\eps) n_+(x)/(|x|+s) ) \;,
$$
and
$$
\D_-(x) = \Theta(\eta g(x) (s^2/\eps) n_-(x)/(|x|+s) ) \;.
$$
\end{corollary}
\begin{proof}
This follows from the explicit formulas for $\D_+$ and $\D_-$ given in Lemma~\ref{lem:formulas-d}
along with Corollary \ref{mSizeCor}, which implies that the denominators $|n+1/2|$ are $\Theta((|x|+s)/s)$.
\end{proof}

We next need to approximate the size of $n_+(x)$ and $n_-(x)$.
We have the following lemma.

\begin{lemma}\label{nSizeLem}
For all $x \in \R$, we have that
$$n_{+}(x), n_{-}(x) = |x| \left( 1/s-1/(s+\eps) \right)+O(1) = |x| \, \Theta(\eps/s^2)+O(1) \;.$$
\end{lemma}
\begin{proof}
We will prove the desired statement for $x\geq 0$ and $n_+(x)$. 
The other cases follow symmetrically. 
Note that $x\in I_{+}^m$ only for non-negative $m$. 
For such $m$, $x\in I_{+}^m$ if and only if $ms \leq x \leq m(s+\eps)+\eps/2$. 
This is the difference in the number of $m$'s for which $ms \leq x$ and 
the number of $m$'s for which $m(s+\eps)+\eps/2 < x$. 
The former is $|x|/s+O(1)$ and the latter is $|x|/(s+\eps)+O(1)$. 
The lemma follows.
\end{proof}

Lemma~\ref{nSizeLem} implies the following.

\begin{corollary}\label{JSupportCor}
We have that $n_+(x) \geq 1$ for all $x$ larger than a sufficiently large constant multiple of $s^2/\eps$.
\end{corollary}

Combining Lemma \ref{nSizeLem} with Corollary \ref{DSizeCor}, we get the following.

\begin{corollary}\label{DBoundCor}
For all $x \in \R$, we have that
$$
\D_+(x) = O \left( g(x) (s/\eps) \right), \ \ \ \D_{-}(x) = O \left(g(x) \, \eta \, (s/\eps) \right) \;.
$$
\end{corollary}

A combination of Lemma \ref{nSizeLem} with Corollary \ref{DSizeCor} 
also implies that on the support of $\D_+$ the ratio $\D_+/\D_-$ is sufficiently large.

\begin{corollary}\label{DRatioCor}
If $x \in \R$ is such that $\D_+(x) > 0$, then $\D_+(x)/\D_-(x) > (1-\eta)/\eta.$
\end{corollary}
\begin{proof}
If $\D_+(x)>0$, then $n_+(x) > 0$. Lemma \ref{nSizeLem} implies that $n_-(x) = n_+(x)+O(1)$, 
and therefore $n_-(x)/n_+(x) = O(1)$. Combining this with Corollary \ref{DSizeCor}, we have that
$$
\D_+(x) / \D_-(x)  = \new{\Omega} (C/\eta) \;.
$$
For $C$ a sufficiently large universal constant, this implies our result.
\end{proof}

We also need to show that the intersection of the supports of $\D_{+}$ and $\D_{-}$
occurs only for $|x|$ sufficiently large. Specifically, we have the following lemma.

\begin{lemma}\label{supportIntersectLem}
For $x \in \R$, it holds that
$\min(n_+(x),n_-(x))>0$ only if $|x| = \Omega(s^2/\eps)$.
\end{lemma}
\begin{proof}
We have that $\min(n_+(x),n_-(x))>0$ only if there exist integers $m$ and $m'$ with 
$x\in I_{+}^m \cap I_{-}^{m'}$. By Corollary \ref{mSizeCor}, it must be the case that 
$|m|,|m'| = O(|x|/s + 1).$ On the other hand, we have that $I_{+}^m$ is an interval 
containing the point $ms$, and $I_{-}^{m'}$ is an interval containing the point $(m'+1/2)s$. 
These points must differ by at least $s/2$, and therefore the sum of the lengths 
of these intervals must be at least $s/2$. On the other hand, these intervals have length 
$|m+1/2|\eps$ and $|m'+1/2|\eps$ respectively. Thus, $\min(n_+(x),n_-(x))>0$ 
can only occur if $s/2 = O(|x|\eps/s+\eps)$, which implies that 
$x=\Omega(s^2/\eps)$, as desired.
\end{proof}

\begin{definition} \label{def:J}
We define $J$ to be $\R \backslash \bigcup_{m\in \Z} I^m_+$. 
\end{definition}

We note that $J$ is a union of intervals. 

\begin{lemma}
We have that $J$ is a union of $O(s/\eps)$ many intervals.
\end{lemma}
\begin{proof}
By Corollary \ref{JSupportCor}, $J$ is an interval $J_0 = [-O(s^2/\eps),O(s^2/\eps)]$ minus 
all of the intervals $I^m_+$ that intersect $J_0$. 
By Corollary \ref{mSizeCor}, $I_{+}^m$ intersects $J_0$ only when $|m| = O(s/\eps)$. 
Thus, $J$ is an interval minus a union of $O(s/\eps)$ other intervals. 
Thus, it is a union of $O(s/\eps)$ many intervals.
\end{proof}

We can now directly verify the properties of Proposition \ref{mainProp}.
\begin{itemize}
\item The definition of $J$ implies that $n_{+}(x)=0$ on $J$, which itself implies that $\D_{+}(x)=0$ for $x\in J$. The latter fact 
combined with Corollary \ref{DRatioCor} imply Property~\ref{prop:1}.

\item Lemma \ref{supportIntersectLem} implies that the intersection of $J^c$ with the support of $\D_-$ 
consists only of points $x$ with $|x| = \Omega(s^2/\eps)$. 
This fact and Corollary \ref{DBoundCor} imply Property~\ref{prop:2} \new{(by Gaussian concentration)}.

\item Property~\ref{prop:3} follows from Lemma \ref{DMomentsLem}.

\item Property~\ref{prop:4} follows from Corollary \ref{DBoundCor}.

\item Property~\ref{prop:5} follows from Lemma \ref{DL1Lem}.
\end{itemize}
This completes the proof of Proposition~\ref{mainProp}.
\end{proof}

\subsection{Putting Everything Together: Proof of Theorem~\ref{thm:main-formal}} \label{ssec:combine}
We now have the necessary ingredients to complete the proof of Theorem~\ref{thm:main-formal}.

\begin{proof}[Proof of Theorem~\ref{thm:main-formal}]
The proof leverages the SQ framework of Section~\ref{ssec:high-dim}
combined with the one-dimensional construction of Proposition \ref{mainProp}.

\paragraph{Parameter Setting}
Recall the parameters in the theorem statement.
We have that $\opt>0$ and $M \in \Z_+$ are such that
$\log(M)/(\log\log(M))^3$ is at least a sufficiently large constant multiple of $\log(1/\opt)$.
Moreover, we define a parameter 
$\tau$ which is set to $M^{-\new{\Theta}\left( \frac{\log(M)}{{\log\log(M)}^3}/\log(1/\opt) \right)}$,
\new{where the implied constant in the exponent is sufficiently small}.

Let $C>0$ be a sufficiently large universal constant.
We define positive integers $m$ and $d$ as follows:
$m = {\lceil C\log(1/\tau) \rceil}$
and $d =  {\lceil  C\sqrt{\log(1/\opt)\log(1/\tau)\log\log(1/\tau)} \rceil}$.
Observe that
\begin{equation}\label{eqn:binom}
\binom{2d+m}{m} \leq m^{2d} =\exp(O(C\sqrt{\log(1/\opt)\log(1/\tau)(\log\log(1/\tau))^3})) \;.
\end{equation}
We note that if $\log(1/\tau)$ is a sufficiently small constant multiple of 
$\frac{\log^2(M)}{(\log\log(M))^3 \log(1/\opt)}$,
then the RHS of~\eqref{eqn:binom} is less than $M$. Thus, by decreasing $M$ if necessary,
we can assume that $M=\binom{2d+m}{m}$.
Consider the Veronese mapping, denote by ${V}_{2d}:\R^m\rightarrow \R^M$,
such that the coordinate functions of ${V}_{2d}$ are exactly the monomials
in $m$ variables of degree at most $2d$.

\paragraph{Hard Distributions}
We can now formally construct the {family of high-dimensional distributions on labeled examples
that (1) corresponds to Massart halfspaces, and (2) is SQ-hard to learn}.
We define univariate measures $\D_{+}$ and $\D_{-}$ on $\R$,
as given by Proposition \ref{mainProp}, 
with $s$ and $\eps$ picked so that $s^2/\eps$ is a sufficiently large constant multiple 
of $\sqrt{\log(1/\opt)}$ and $s/\eps$ a sufficiently small constant multiple of $d$ 
(for example, by taking $s=C^2 \sqrt{\log(1/\opt)}/d = \Theta(1/\sqrt{\log(1/\tau)\log\log(1/\tau)})$ 
and $\eps = C^3 \sqrt{\log(1/\opt)}/d^2$). 

For a unit vector $v\in \R^m$, consider the distribution $\p_v^{\D_{+},\D_{-},p}$,
as in Proposition~\ref{prop:generic-sq}, with $p=\|\D_{+}\|_1/(\|\D_{+}\|_1 + \|\D_{-} \|_1)$.
\new{By property~\ref{prop:5} of Proposition~\ref{mainProp}, 
note that $\min (p, 1-p) = 1-p = \Theta(\eta)$.}

Our hard distribution is the distribution $(X',Y')$ on $\R^M \times \{\pm 1\}$ obtained
by drawing $(X,Y)$ from $\p_v^{\D_{+},\D_{-},p}$ and letting $X'={V}_{2d}(X)$ and $Y'=Y$.

We start by showing that this corresponds to a Massart halfspace.

\begin{claim}\label{clm:pv-Massart-LTF}
The distribution $(X',Y')$ on $\R^M \times \{\pm 1\}$  is a Massart LTF distribution
with optimal misclassification error $\opt$ and Massart noise rate upper bound of $\eta$.
\end{claim}
\begin{proof}
{For a unit vector $v \in \R^m$,} let $g_{{v}}: \R^m \to \{\pm 1\}$ be defined as
$g_{{v}}(x)=-1$ if and only if $v \cdot x \in J$, {where $J$ is the union of intervals in the construction
of Proposition~\ref{mainProp}}. Note that $g_{{v}}$ is a degree-$2d$ PTF {on $\R^m$},
{since $g_v$ is a $(2d+1)$-piecewise constant function of $v \cdot x$}.
Therefore, there exists some LTF $L:\R^M \to \{\pm 1\}$ such that
$g_{{v}}(x)=L({V}_{2d}(x))$ for all $x \in \R^m$.

Note that our hard distribution returns $(X',Y')$ with $Y'=L(X')$,
unless it picked a sample corresponding to a sample of $\D_{-}$ coming from $J^c$,
which happens with probability at most $\zeta < \opt$.
Additionally, suppose that our distribution returned a sample with $X'={V}_{2d}(X)$, for some $X\in \R^m$.
By construction, conditioned on this {event}, we have that $Y'=1$
with probability proportional to $\D_{+}(v\cdot X)$,
and $Y'=-1$ with probability proportional to $\D_{-}(v\cdot X)$.
We note that if $L({V}_{2d}(X))=1$, then $v\cdot X \not\in J$;
so, by Proposition~\ref{mainProp} property~\ref{prop:1}(b), this ratio is at least $1-\eta:\eta$.
On the other hand, if $L({V}_{2d}(X))=-1$, then $v\cdot X\in J$, so $\D_{+}(v\cdot X)=0$.
This implies that the pointwise probability of error $\eta(X')$ is at most \new{$\eta$}, 
completing the proof of the claim.
\end{proof}

We are now ready to complete the proof of our SQ lower bound.
It is easy to see that finding a hypothesis that predicts $Y'$ given $X'$
is equivalent to finding a hypothesis for $Y$ given $X$
(since $Y=Y'$ and there is a known 1-1 mapping between $X$ and $X'$).
The pointwise bounds on $\D_{+}$ and $\D_{-}$, {specifically properties~\ref{prop:4} and~\ref{prop:5}
in Proposition~\ref{mainProp},}
imply that
$$\chi^2(\D_{\pm}/\|\D_{\pm}\|_1 \, , \, G) \leq O(s/\eps)^2 = \mathrm{polylog}(M) \;.$$
The parameter $\nu$ in Proposition \ref{prop:generic-sq} is
$k!\exp(-\Omega(1/s^2)) = \exp(-\Omega(1/s^2))$ after 
taking $k$ to be a sufficiently small constant multiple $\log(1/s)/s^2$.

Thus, by Proposition \ref{prop:generic-sq}, \new{in order to output a hypothesis with
error smaller than $\min(p, 1-p) = \Theta(\eta)$,}
any SQ algorithm either needs queries with accuracy better than
$$
\nu^2 + 2^{-k}(\chi^2(A,G)+\chi^2(B,G)) = \exp(-\Omega(\log(1/s)/s^2))\mathrm{polylog}(M) < \tau
$$
or a number of queries more than
$$
2^{\Omega(m)}\tau (\chi^2(A,G)+\chi^2(B,G)) > 1/\tau \;.
$$
Therefore, Proposition \ref{prop:generic-sq} implies that it is impossible for an SQ algorithm
to learn a hypothesis with error better than $\Theta(\eta)$
without either using queries of accuracy better than $\tau$
or making at least $1/\tau$ many queries.
This completes the proof of Theorem~\ref{thm:main-formal}.
\end{proof}

\subsection{Obtaining Optimal Error: The Case of Large $\eta$} \label{ssec:large-eta}

\new{
In this final subsection, we refine the construction of the previous subsections to obtain
a sharp lower bound of $\eta-o_{\eta}(1)$, when $\eta$ is close to $1/2$. 
Here the term $o_{\eta}(1)$ goes to zero when $\eta$ approaches $1/2$. }
Specifically, we show:

\begin{theorem}[Sharp SQ Hardness of Massart Halfspaces for Large $\eta$] \label{thm:main-large-eta}
Let $\opt>0$ and $M \in \Z_+$ be such that
$\log(M)/(\log\log(M))^3$ is at least a sufficiently large constant multiple of $\log(1/\opt)$.
\new{Let $c>0$ be any parameter such that $c \gg \sqrt{1/2-\eta}$.}
There exists a parameter $\tau \eqdef M^{-\Omega_{\new{c}}\left( \frac{\log(M)}{{\log\log(M)}^3}/\log(1/\opt) \right)}$
such that no SQ algorithm can learn the class of halfspaces on $\R^M$
in the presence of $\eta$-Massart noise, \new{where $\opt< \eta \leq 1/2$,} within
error better than $1/2 - \new{c}$ using at most $1/\tau$ queries of tolerance $\tau$. 
This holds even if the optimal classifier has misclassification error at most $\opt$.
\end{theorem}

\new{Conceptually, Theorem~\ref{thm:main-large-eta} provides evidence that {\em even the constant factor} 
(of $1$) in the error guarantee (of $\eta+\eps$) achieved by the Massart learner of~\cite{DGT19} 
cannot be improved in general.}

The proof of Theorem~\ref{thm:main-large-eta} proceeds along the same lines
as the proof of Theorem~\ref{thm:main-formal}. 
The main difference is in the choice of the one-dimensional moment-matching distributions. 
For this, we use a construction that is qualitatively similar (though 
somewhat more sophisticated) 
to that used in the proof of Section~\ref{ssec:one-dim}.

Specifically, for some carefully chosen parameter $C>0$ (to be determined),
we define the positive measures:
$$
\D_+ := C \, (s/\eps) \,  \int_{0}^{\eps} \frac{G_{s+y, y/2}}{s+y} dy \;,
$$
and
$$
\D_- := (s/\eps) \, \int_{0}^{\eps} \frac{G_{s+y,(y+s)/2}}{s+y} dy \;.
$$ 
As a refinement of Corollary \ref{DSizeCor}, we obtain the following. 
\begin{corollary} For all $x \in \R$, we have that
$$
\D_+(x) = C g(x) (s^2/\eps) n_+(x)/(|x|+1)(1+O(\eps/(|x|+1))) 
$$
and
$$
\D_-(x) = g(x) (s^2/\eps) n_-(x)/(|x|+1)(1+O(\eps/(|x|+1))) \;.
$$
\end{corollary}
\begin{proof}
This follows from the explicit formulas for $\D_+$ and $\D_-$ (Lemma~\ref{lem:formulas-d})
along with the fact that for 
$x\in I_m^{\pm}$, $1/|m|$ and $1/|m+1/2|$ are $s/(|x|+1)(1+O(\eps/(|x|+1)))$.
\end{proof}

Using the above corollary, we obtain the following.

\begin{corollary}
For all $x \in \R$, we have that 
$\D_+(x)/\D_-(x) = C \, q(x)(1+O(\eps/(|x|+1)))$, where 
$q(x) = n_{+}(x)/n_{-}(x)$ is a rational number 
with numerator and denominator at most $O(|x|/s+1)$.
\end{corollary}

We want to guarantee that for all $x\in \R$ it holds that 
$\D_+(x)/\D_-(x) \not \in [\eta/(1-\eta),(1-\eta)/\eta]$. 
We note that this condition automatically holds for $|x|$ less than a sufficiently 
small constant multiple of $s^2/\eps$, 
as in this range we have that $\min(n_+(x),n_-(x))=0$. 
For points $x$ outside this range, we have that 
$\D_{+}(x)/\D_{-}(x) = C \, q(x) (1+O(\eps/s)^2)$. 
Furthermore, since $|n_+(x)-n_-(x)|\leq 1$, the latter 
implies that in this range $\D_+(x)/\D_{-}(x)$ is always one of:
\begin{itemize}
\item $C(1+O(\eps/s)^2)$ \;,
\item $C(1+1/m)(1+O(\eps/s)^2)$, for some integer $m$, 
\item $C(1-1/m)(1+O(\eps/s)^2)$, for some integer $m$.
\end{itemize}
We will arrange that this quantity is always in the appropriate range 
by picking the parameter $C$, 
so that for some well chosen $m_0$ 
we have that
$$
C(1-1/m_0)(1+O(\eps/s)^2) \leq \eta/(1-\eta) \textrm{, and   } C(1-1/(m_0+1))(1-O(\eps/s)^2) \geq (1-\eta)/\eta \;.
$$
If the above holds, it is easy to see that $\D_+(x)/\D_{-}(x)$ will never be in the range 
$[\eta/(1-\eta),(1-\eta)/\eta]$ for any value of $x$. In order to arrange this, we set $C$ to satisfy
$$
C(1-1/m_0)(1+O(\eps/s)^2) = \eta/(1-\eta) \;.
$$
In order for the second condition to hold, it must be the case that
$$
\left(\frac{1-1/(m_0+1)}{1-1/m_0} \right)(1-O(\eps/s)^2) > ((1-\eta)/\eta)^2 = 1+O(1/2-\eta) \;.
$$
For the latter to be true, it must hold that $1/m_0^2 $ is at least 
a sufficiently large constant multiple of $(1/2-\eta)+(\eps/s)^2$, 
or that $m_0$ is at most a sufficiently small constant 
multiple of $\min(s/\eps, \sqrt{1/2-\eta})$.

In particular, if we take $m_0$ to be \new{at most} a sufficiently small constant 
multiple of $1/\sqrt{1/2-\eta}$ and ensure that $\eps/s$ is sufficiently small, 
this construction can be made to work with $C= 1+ \new{1/m_0}$. 

We then let $J$ be the set of points $x \in \R$ for which $\D_{-}(x) > \D_{+}(x)$. 
It is easy to see that $J =\{x: m_0 \geq n_-(x) > n_+(x)\}$, 
and from this it can be seen that $J$ is a union of $O(m_0 s/\eps)$ intervals. 
As before, $\D_{+}$ and $\D_{-}$ approximately match many moments 
with a Gaussian and the mass of $\D_{+}$ on $J$ and $\D_{-}$ on $J^c$ 
are both supported on points $x$ such that 
$|x| \geq \Omega(s^2/\eps)$, and thus have mass $\exp(-\Omega(s^4/\eps^2))$. 

Furthermore, we have that $\D_{-}(x)/\D_{+}(x) > (1-\eta)/\eta$ for $x \in J$ and 
$\D_{+}(x)/\D_{-}(x) > (1-\eta)/\eta$ for $x \in J^c$. Therefore, the appropriate hidden-direction 
distribution is a degree-$O(m_0 s/\eps)$ PTF with at most $\eta$ Massart noise.

Finally, it is not hard to see that $\|\D_+\|_1/\|\D_-\|_1 = 1+\new{1/m_0}$. 
Therefore, \new{by following the arguments of Section~\ref{ssec:combine} mutatis-mutandis},
it follows that for any constant $\eta<1/2$ it is SQ-hard to learn an LTF with $\eta$-Massart noise 
to error better than \new{$1/2-c$ for any $c \gg \sqrt{1/2-\eta}$}, 
even when $\opt$ is almost polynomially small in the dimension.
This completes the proof of Theorem~\ref{thm:main-large-eta}. \qed

\section{Conclusions and Future Work} \label{sec:conc}

In this paper, we gave a super-polynomial Statistical Query (SQ) lower bound
with near-optimal inapproximability gap for the fundamental problem of 
(distribution-free) PAC learning Massart halfspaces. 
Our lower bound provides strong evidence that known algorithms 
for this problem are essentially best possible. 
An obvious technical open question is whether the hidden constant factor 
in the $\Omega(\eta)$-term of our lower bound 
can be improved to the value $C=1$ for all $\eta>0$.
(Recall that we have shown such a bound for $\eta$ close to $1/2$.) 
This would match known algorithms {\em exactly}, specifically showing that the error of $\eta+\eps$
cannot be improved even for small values of $\eta>0$.

Interestingly, SQ lower bounds are the {\em only known} evidence of hardness 
for our Massart halfspace learning problem. Via a recent reduction~\cite{BBHLS20}, our SQ
lower bound implies a similar low-degree polynomial testing lower bound for the problem. 
An interesting open question is to prove similar hardness results against
families of convex programming relaxations (obtained, e.g., via the Sum-of-Squares framework).
Such lower bounds would likely depend on the underlying optimization formulation of the learning problem. 

A related question is whether one can establish reduction-based computational hardness for learning
halfspaces in the presence of Massart noise. Daniely~\cite{Daniely16} gave such a reduction for the (much more challenging) problem of agnostically learning halfspaces, 
starting from the problem of strongly refuting random XOR formulas. 
It currently remains unclear whether the latter problem is an appropriate starting point 
for proving hardness in the Massart model. That said, obtaining reduction-based hardness 
for learning Massart halfspaces is left as an interesting open problem.

\clearpage

\bibliographystyle{alpha}
\bibliography{allrefs}

\newpage

\appendix

\section*{Appendix}

\section{Proof of Lemma~\ref{lem:cor}}\label{app:cor}

Let $\theta$ be the angle between $v$ and $v'$.
By making an orthogonal change of variables, we can reduce to the case
where $v=(1,0,\ldots,0)$ and $v'=(\cos(\theta),\sin(\theta),0,0,\ldots,0)$.
Then by definition we have that $\chi_{N(0,I)}(\p_v, \p_{v'})+1$ is
$$
\int_{\R^m} \left( \frac{A(x_1)A(\cos(\theta)x_1+\sin(\theta)x_2)g(x_2)g(\sin(\theta)x_1-\cos(\theta)x_2)}{g(x_1)g(x_2)} \right)g(x_3)\cdots g(x_m) dx_1\cdots dx_m \;.
$$
Noting that the integral over $x_3,\ldots,x_m$ separates out,
we are left with
$$
\int_{\R^2} \left( \frac{A(x)A(\cos(\theta)x+\sin(\theta)y)g(y)g(\sin(\theta)x-\cos(\theta)y)}{g(x)g(y)} \right)dxdy \;.
$$
Integrating over $y$ gives
$$
\int \frac{A(x)}{g(x)}\left(\int  A(\cos(\theta)x+\sin(\theta)y)g(\sin(\theta)x-\cos(\theta)y) dy\right) dx = \int \frac{A(x)U_{\cos(\theta)} A(x)}{g(x)}dx \;,
$$
where $U_t$ is the Ornstein-Uhlenbeck operator.
We will simplify our computations by expressing the various quantities
in terms of the eigenbasis for this operator.

In particular, let $h_n(x) = He_n(x)/\sqrt{n!}$ where $He_n(x)$ is the probabilist's Hermite polynomial.
We note the following basic facts about them:
\begin{enumerate}
\item $\int_\R h_i(x)h_j(x)g(x)dx = \delta_{i,j}.$
\item $U_t (h_n(x)g(x)) = t^n h_n(x)g(x).$
\end{enumerate}
We can now write $A(x)$ in this basis as
$$
A(x) = \sum_{n=0}^\infty a_n h_n(x)g(x) \;.
$$
From this, we obtain that
\begin{align*}
\chi^2(A,N(0,1)) &= \int_\R \left( \sum_{n=0}^\infty a_n h_n(x)g(x)\right)^2/g(x)dx\\
& = \int_\R \sum_{n,m=0}^\infty a_n a_m h_n(x)h_m(x) g(x)dx\\
& = \sum_{n=0}^\infty a_n^2 \;.
\end{align*}
Furthermore, we have that
$$
\int_\R h_s(x)A(x)dx = \int_\R\sum_{n=0}^\infty a_n h_s(x)h_n(x)g(x)dx = a_s \;.
$$
For $1\leq s\leq k$, we have that
$$
h_s(x)=\sqrt{s!}\sum_{t=0}^{\lfloor s/2\rfloor} \frac{(-1)^tx^{s-2t}}{2^t t!(n-2t)!} \;.
$$
We therefore have that
$$
a_s = \sum_{t=0}^{\lfloor s/2\rfloor}\left( \frac{\sqrt{s!}(-1)^tx^{s-2t}}{2^t t!(s-2t)!} \right) \E[A^{s-2t}] \;.
$$
Note that the above is close to
$$
\sum_{t=0}^{\lfloor s/2\rfloor}\left( \frac{\sqrt{s!}(-1)^tx^{s-2t}}{2^t t!(s-2t)!} \right) \E[G^{s-2t}] = \E[h_s(G)] = 0 \;.
$$
In particular, the difference between the two quantities is at most
$$
\nu \, \sum_{t=0}^{\lfloor s/2\rfloor}\left( \frac{\sqrt{s!} }{2^t t!(s-2t)!} \right) \;.
$$
It is easy to see that the denominator is minimized when $t = s/2 - O(\sqrt{s})$.
From this it follows that this sum is $2^{O(s)} \, \nu$.
Therefore, we have that $a_s=2^{O(s)} \, \nu$, for $1\leq s\leq k$.
Furthermore, $a_0 = \int A(x)dx = 1$. Thus, we have that
\begin{align*}
\chi_{N(0,I)}(\p_v, \p_{v'})+1 & = \int_\R \frac{A(x)U_{v\cdot v'}A(x)}{g(x)} dx\\
& = \int_\R \left(\sum_{n=0}^\infty a_n h_n(x) g(x)\right)\left(\sum_{n'=0}^\infty a_n' (v\cdot v')^{n'} h_n'(x) g(x)\right)/g(x) dx\\
& = \int_\R \sum_{n,n'=0}^\infty a_n a_n' (v\cdot v')^{n'} h_n(x) h_n'(x) g(x) dx\\
& = \sum_{n=0}^\infty a_n^2 (v\cdot v')^n\\
& = 1 + \sum_{n=1}^k a_n^2 (v\cdot v')^{n} + \sum_{n=k+1}^\infty a_n^2 (v\cdot v')^n \;.
\end{align*}
Therefore,
\begin{align*}
|\chi_{N(0,I)}(\p_v, \p_{v'})|
& \leq O(\nu^2)\sum_{n=1}^k 2^{O(n)} \, |v\cdot v'|^{n} + |v\cdot v'|^{k+1}\sum_{n=0}^\infty a_n^2\\
& \leq \nu^2 + |v\cdot v'|^{k+1}\chi^2(A, N(0,1)) \;.
\end{align*}
This completes our proof.

\section{Proof of Lemma~\ref{lem:disc-Gaussian-mm}}\label{app:disc-Gaussian}

We consider the Fourier transform of $G_{\sigma,\theta}$.
Note that $G_{\sigma,\theta}$ is the pointwise product of $G$ with a mesh of delta-functions.
Therefore, its Fourier transform is the convolution of their Fourier transforms.
The Fourier transform of $G$ is $\sqrt{2\pi}G$. The Fourier transform of the net of delta-functions $f(\xi) = \sum_{n\in\Z}\delta(\xi-n/\sigma) e^{2\pi i \theta \xi}$. Thus, we have that the Fourier transform of $G_{\sigma,\theta}$ at $\xi$ is
$$
\sum_{n\in \Z} \sqrt{2\pi}g(\xi+n/\sigma)e^{-2\pi i n \theta/ \sigma} \;.
$$
The $t^{th}$ moment of a pseudodistribution is proportional to the value of the $t^{th}$ derivative of its Fourier transform at $\xi=0$. For $G$, this is $\sqrt{2\pi}g^{(t)}(0)$. 
For $G_{\sigma,\theta}$, it is equal to this term plus
$$
\sum_{n\in \Z, n \neq 0}\sqrt{2\pi}g^{(t)}(n/\sigma)e^{-2\pi i n \theta/ \sigma} \;.
$$
Computing the derivative of $g$ using Cauchy's integral formula (integrating around a circle of radius $1/(2\new{\sigma})$ centered at $n/\new{\sigma}$), we find that
$$
|g^{(t)}(n/\sigma)| = t! O(\sigma)^{t} \exp(-\Omega(n/\sigma)^2).
$$
Taking a sum over $n$ yields our result.

\end{document}